\documentclass{article}

\usepackage{iclr2026_conference,times}

\usepackage{amsmath,amssymb,amsthm,mathtools}
\usepackage{centernot}
\usepackage{enumitem}
\usepackage{graphicx}
\usepackage{hyperref}
\usepackage{microtype}
\usepackage{tikz}
\usepackage{times}
\usepackage{url}
\usepackage[dvipsnames,table]{xcolor}

\usetikzlibrary{arrows.meta,positioning,calc}

\hypersetup{
  colorlinks=true,
  linkcolor=RubineRed,
  citecolor=Emerald,
  urlcolor=RubineRed
}

\newtheorem{definition}{Definition}
\newtheorem{theorem}{Theorem}
\newtheorem{proposition}{Proposition}
\newtheorem{corollary}{Corollary}
\newtheorem{remark}{Remark}

\DeclareMathOperator*{\infimum}{inf}

\newcommand{\X}{\mathcal{X}}

\newcommand{\Gam}{\Gamma}

\newcommand{\Act}{\mathsf{Act}}

\title{
\fontsize{17.5}{22.0}\selectfont
Intrinsic-Energy Joint Embedding Predictive\\
Architectures Induce Quasimetric Spaces
}

\iclrfinalcopy

\author{
Anthony Kobanda$^{1,2}$,
Waris Radji$^2$\\
$^1$Ubisoft La Forge, Bordeaux, France,\\
$^2$Inria Scool, 
Université de Lille, France.
}

\begin{document}

    \maketitle

    % FRONT MATTER

        \vspace{-0.75em}

        % ABSTRACT
        \begin{abstract}
Joint-Embedding Predictive Architectures (JEPAs) aim to learn representations by predicting target embeddings from context embeddings, 
inducing a scalar\linebreak
compatibility energy in a latent space.
In contrast, Quasimetric Reinforcement Learning (QRL) studies goal-conditioned control through \emph{directed} distance values\linebreak
(cost-to-go) that support reaching goals under asymmetric dynamics. 
In this short article,
we connect these viewpoints by restricting attention to a principled class of JEPA energy functions : \emph{intrinsic (least-action) energies}, defined as infima of\linebreak
accumulated local effort over admissible trajectories between two states.
Under\linebreak
mild closure and additivity assumptions, any intrinsic energy is a quasimetric.
In goal-reaching control, optimal cost-to-go functions admit exactly this intrinsic form ; 
inversely, JEPAs trained to model intrinsic energies lie in the quasimetric value class targeted by QRL. Moreover, we observe why symmetric finite energies are structurally mismatched with one-way reachability, motivating asymmetric (quasimetric) energies when directionality matters.
\end{abstract}

        \vspace{-0.75em}

        % INTRODUCTION
        \section{Introduction}
\textbf{Joint-Embedding Predictive Architectures} (\textbf{JEPAs}) \citep{jepa_lecun_2022,assran2023self} recently emerged as a compelling self-supervised paradigm for representation learning in which models\linebreak
predict \emph{target embeddings} from \emph{context embeddings} rather than reconstructing raw observations. 
From a control viewpoint, the resulting training score can be interpreted as a scalar \emph{compatibility\linebreak
energy} between inputs : low energy meaning \textit{compatible},
while high energy means \textit{incompatible}.

Independently, \textbf{Quasimetric Reinforcement Learning} (\textbf{QRL}) \citep{qrl}
represents goal-conditioned control through the geometry of a learned cost-to-go. 
A central fact is that reaching costs induce a \emph{directed} notion of distance :
going from $s$ to $g$ can be easy, while returning from $g$ to $s$ may be difficult or impossible, and multi-step composition should be consistent. 
Hence, QRL formalizes this by targeting function classes with \emph{quasimetric} structure \citep{iqe}, and proposes a learning framework and objectives designed for that regime.

This short article addresses a narrow structural question :
\begin{quote}
\emph{Under what conditions does a JEPA-induced energy behave like a cost-to-go ?}
\end{quote}
A key insight is that sequential compositionality is not a modelling preference, but usually required when reasoning \citep{mentalmodel,llmmultistep}.
If a state $u$ (e.g., a position, a proposition, a concept) is accessible 
(e.g., reachable, demonstrable, explainable) on the way from $x$ to $y$, then an energy meant to support multi-step reasoning should satisfy a consistency inequality of the form :
$$
E(x,g) \;\le\; E(x,u) + E(u,g),
$$
which is precisely a triangle inequality. 

In physics and optimal control, the canonical way to obtain such a global, compositional energy is to define it intrinsically via a \emph{least-action} principle: the energy between two states is the infimum of accumulated local effort over admissible trajectories connecting them \citep{robotlap,phylap,metalap}. 
This construction yields triangle inequality by definition (via concatenation), and it naturally produces asymmetry whenever admissibility or local effort is directed.\linebreak 

Our main message is simple :
\begin{quote}
\begin{center}
\emph{If a JEPA is defined according to an intrinsic
(least-action) energy,\\
then it induces a quasimetric space.}
\end{center}
\end{quote}
In goal-conditioned tasks, costs-to-go can be interpreted as intrinsic energies, hence such \emph{Intrinsic-Energy JEPAs} fall into the same quasimetric value-functions class that QRL is designed to model.

        % RELATED WORK
        \section{Minimal Background}

\paragraph{JEPA (Latent Representations and Induced Energies).}
Joint-Embedding Predictive Architectures learn representations by predicting \emph{target embeddings} from \emph{context embeddings} \citep{jepa_lecun_2022}. 
Given a context encoder $f_\phi$ and a target encoder $f_{\bar\phi}$, JEPA forms
$
z_x=f_\phi(x), z_y=f_{\bar\phi}(y),
$
and a predictor $p_\theta$ produces a predicted target embedding $\hat z_y=p_\theta(z_x;\,c)$, where $c$ denotes a conditioning. 

Training uses a comparator function $D(\cdot,\cdot)$ in the embedding space :
$$
\mathcal{L}_{\text{JEPA}}(\phi,\bar\phi,\theta)
~=~\mathbb{E}_{(x,y)\sim\mathcal{D}}\!\left[\,D\!\left(\hat z_y,\texttt{sg}(z_y)\right)\right],
$$
Even when the comparator form is fixed, the model learns an \emph{induced energy landscape} over inputs through the learned representation and predictor: pairs $(x,y)$ that are \textit{compatible} are precisely those for which the latent prediction error is small. This motivates interpreting JEPA training as learning a scalar energy (compatibility) between $x$ and $y$, in line with the energy-based perspective on JEPA.

\paragraph{QRL (Goal-Reaching values are Quasimetrics).}
Quasimetric Reinforcement Learning studies goal-conditioned control through directed distances that compose over time \citep{qrl}. 
In reaching-cost problems, the optimal value function is $V^\star(s,g)$, usually negative (up to sign conventions), and satisfies a triangle inequality in $(s,g)$, yielding a quasimetric structure :
$$
d^\star(s,g) \;\triangleq\; -V^\star(s,g),\qquad 
d^\star(s,g)\le d^\star(s,w)+d^\star(w,g),
$$
with $d^\star(s,g)=+\infty$ for unreachable pairs. QRL leverages this structure by learning $d_\theta$ within quasimetric function classes \citep{iqe}, enforcing local constraints from observed transitions and using the triangle inequality to propagate these constraints to long horizons; the resulting objectives come with recovery guarantees under suitable data coverage conditions. For our purposes, the key takeaway is that QRL treats the goal-conditioned value as a \emph{directed geometry}.

\paragraph{Value-Guided Action Planning with JEPA World Models.}
Value-guided JEPA planning uses JEPA-like world models for control by shaping representation spaces so that an embedding-space cost aligns with a goal-reaching value, enabling planning via minimizing a representation-space objective \citep{destrade2025value}. 
Their contribution is primarily algorithmic and empirical: learning representations that make planning effective.
Our focus is orthogonal and structural: we do not propose planners or additional experiments.
Instead, we isolate a hypothesis-class condition on the \emph{energy} itself under which a JEPA-induced energy necessarily satisfies quasimetric inequalities, directly connecting JEPA energies to the quasimetric value viewpoint formalized in QRL.

\paragraph{Intrinsic Energies and the Least-Action Principle.}
In physics and control,
it is standard to define a global cost between configurations as the infimum of an \emph{action functional},
typically an integral of a local effort along admissible trajectories, following a least-action viewpoint \citep{siburg2004principle}.\linebreak
This perspective is used in modern ML to endow learned representations with physically meaningful structure
\citep{greydanus2019hamiltonian,cranmer2020lagrangian}, 
and more recently to motivate learning from variational constructions
\citep{guo2025physics}.
In optimal control, path-integral methods also build directly on trajectory functionals 
of the form $\int_0^T \ell(x(t),u(t))\,dt$, optimizing them to produce actions, which makes the ``least accumulated effort over trajectories'' interpretation operational in control pipelines 
\citep{asmar2023model,zhai2025pa}.
In this article, we adopt the same high-level principle but in a representation learning setting: we interpret the JEPA score as an \emph{energy} and restrict attention to energies that admit an intrinsic (least-action) representation.

    % MAIN MATTER

        % MAIN A
        \section{Intrinsic Energy Functions and Quasimetrics}

\begin{definition}[Quasimetric]
A function $d:\X\times\X\to\mathbb{R}$ is a \emph{quasimetric} if, 
for all $x,y,z\in\X$ :\linebreak
(i : reflexivity) $d(x,x)=0$ , 
(ii : non-negativity) $d(x,y)\ge 0$ , 
(iii : Identity of indiscernibles) if $d(x,y)=0$ then $x=y$,
(iv : triangular inequality) $d(x,z)\le d(x,y)+d(y,z)$.
\end{definition}

\begin{definition}[Intrinsic (\textit{Least-Action}) Energy]
Let $\X$ be a path-connected state space.
For $x,y\in\X$, let $\Gam(x\!\to\!y)$ denote a set of admissible $C^1$ trajectories $\gamma:[0,T]\to\X$, with $\gamma(0)=x$ and $\gamma(T)=y$, $\forall T>0$.
Let $L:\mathrm{T}\X\to\bar{\mathbb{R}}^+$ be a local effort density, veryfing $L(x,v)\ge c\cdot\lVert v\rVert$ for some norm $\lVert\cdot\rVert$ and a constant $c>0$. 
Defining the action of a trajectory as :
$$
\Act(\gamma)\;=\;\int_0^T L(\gamma(t),\dot\gamma(t))\,dt.
$$
the intrinsic energy is
$
E(x,y)\;=\;\infimum_{\gamma\in\Gam(x\to y)} \Act(\gamma),
$
with $E(x,y)=+\infty$ if $\ \Gam(x\to y)=\emptyset$.
\end{definition}

\begin{remark}[Physics and Control Grounding]
The least-action form is standard: it defines global energies from local effort and yields Euler--Lagrange dynamics under suitable regularity.
In optimal control, cost-to-go functions arise as infima of accumulated running costs over feasible trajectories.
\end{remark}

\begin{theorem}[Intrinsic Energy is a Quasimetric]\label{thm:intrinsic_quasimetric}
$E$ is a quasimetric on $\X$.
\end{theorem}

\begin{proof}
\underline{Non-negativity :} This property holds since $L\ge 0$.

\underline{Reflexivity :} Considering the constant trajectory $\gamma(t)\equiv x$ it comes that $\dot\gamma(t)=0$,
$L(x,0)$ is independent of $t$, and for any $T$, $\Act(\gamma)=L(x,0)\cdot T$.
Thus $E(x,x)=0$ considering the infimum.

\underline{Identity of indiscernibles :} Let's consider $x,y\in\X$ such that $E(x,y)=0$.
Since for any trajectory $\gamma$ :\\
$$
\Act(\gamma)\;=\;\int_0^T L(\gamma(t),\dot\gamma(t))\,dt
\ge \int_0^T c\cdot\lVert\dot\gamma(t)\lVert\,dt
\ge c\cdot\Big\lVert\int_0^T\dot\gamma(t)\,dt\Big\lVert
\ge c\cdot\Big\lVert\gamma(T)-\gamma(0)\Big\lVert\ .
$$
With $\lVert\gamma(T)-\gamma(0)\lVert= \Vert x-y\lVert$, it comes that $E(x,y)=0\ge c\cdot\lVert x-y\lVert\ge0$, consequently $x=y$.

\underline{Triangle inequality :} Let's consider $x,y,z\in\X$. 
We choose $\gamma_{xy}\in\Gam(x\to y)$,
$\gamma_{yz}\in\Gam(y\to z)$,
and $\varepsilon>0$ 
such that
$\Act(\gamma_{xy})\le E(x,y)+\varepsilon$ and $\Act(\gamma_{yz})\le E(y,z)+\varepsilon$.
Then considering the 
concatenation
$\gamma_{xz}=\gamma_{xy}\star \gamma_{yz}\in\Gam(x\to z)$ and use additivity of the integral, we have :
$$
E(x,z)\le \Act(\gamma_{xz})=\Act(\gamma_{xy})+\Act(\gamma_{yz})\le E(x,y)+E(y,z)+2\varepsilon\ .
$$
Since it is true for any $\varepsilon>0$, with $\varepsilon\to0$ the triangular inequality holds.
\end{proof}

\begin{proposition}[Asymmetry is Generic]\label{prop:asym_generic}
If either (i) \underline{admissibility is directed}, 
i.e., for $x,y\in\X$, 
$f:t\to\gamma(t)\in\Gam(x\to y)$ 
$\centernot\implies$ 
$\bar{f}:t\to\gamma(T-t)\in\Gam(y\to x)$, 
\underline{or}
(ii) \underline{local effort is anisotropic} (e.g., $L(x,v)\neq L(x,-v)$), 
\underline{then in general $E(x,y)\neq E(y,x)$}.
\end{proposition}

\begin{remark}[A Minimal Picture]\label{remark:minimal}
\textit{\color{RubineRed}\hyperref[thm:intrinsic_quasimetric]{Theorem}
\ref{thm:intrinsic_quasimetric}}
formalizes the statement ``valid world energies must compose across time'':
the energy of $x\to z$ cannot exceed the energy of $x\to y$ plus the energy of $y\to z$.
This is the triangle inequality that QRL builds into its value geometry \citep{qrl}.
\end{remark}

        % MAIN B
        \section{Intrinsic-Energy JEPA and the QRL hypothesis class}
\vspace{-0.2em}

\begin{definition}[Intrinsic-Energy JEPA (IE-JEPA)]
Let $f_\phi:\X\to\mathbb{R}^d$ be an encoder, and let $\mathcal{C}$ be a prediction rule producing a scalar score from embeddings (e.g., $\|p_\theta(f_\phi(x))-f_{\bar\phi}(y)\|^2$ as in I-JEPA \citep{assran2023self}).
We say a JEPA induces an energy $E_{\phi,\theta}:\X\times\X\to\bar{\mathbb{R}}^+$ if its evaluation score can be interpreted as $E_{\phi,\theta}(x,y)$.
We call it an \emph{Intrinsic-Energy JEPA} if $E_{\phi,\theta}$,
in the sense of Definition~2, i.e., a least-action energy consistent with concatenation and local effort accumulation.
\end{definition}

\begin{corollary}[IE-JEPA Energies are Quasimetrics]\label{cor:iejepa_qm}
If a JEPA-induced energy $E_{\phi,\theta}$ is intrinsic, then it is a quasimetric.  
\textit{\color{RubineRed}\hyperref[fig:jepa-quasimetric]{Figure}
\ref{fig:jepa-quasimetric}}
illustrates this parallel.
\end{corollary}

\begin{remark}[Does JEPA ``Learn the energy'' ?]
In common JEPA instantiations (e.g., I-JEPA), the comparator is often fixed (e.g., $\ell_2$ in embedding space), while the encoders and predictors are learned. In that sense, the \emph{analytic form} of the comparator can be given, yet the \emph{induced energy landscape} over inputs is learned through the learned representation and predictor.
Conversely, QRL explicitly parameterizes a quasimetric cost and learns it with objectives tailored to a quasimetric structure.
\end{remark}

\begin{figure}[h]
\centering
\begin{tikzpicture}[
    scale=0.8,
    transform shape,
    node distance=2.5cm,
    state/.style={circle, draw=black, thick, fill=blue!15, minimum size=1.2cm, font=\large\bfseries},
    goal/.style={circle, draw=black, thick, fill=green!20, minimum size=1.2cm, font=\large\bfseries},
    intermediate/.style={circle, draw=black, thick, fill=orange!20, minimum size=1.2cm, font=\large\bfseries},
    arrow/.style={-{Stealth[length=3mm]}, thick, draw=blue!60!black},
    dasharrow/.style={-{Stealth[length=3mm]}, thick, draw=red!70!black, dashed},
    label/.style={font=\small, midway, fill=white, inner sep=2pt},
    title/.style={font=\normalsize\bfseries, text=black},
]
% === LEFT SIDE: JEPA Framework ===
\node[title] at (-4, 4.5) {JEPA Energy};
\node[state] (x1) at (-6, 2) {$x$};
\node[intermediate] (u1) at (-4, 3.5) {$u$};
\node[goal] (g1) at (-2, 2) {$g$};
\draw[arrow] (x1) -- (g1) node[label, below, yshift=-2pt] {$E(x,g)$};
\draw[arrow] (x1) -- (u1) node[label, above left, xshift=-2pt] {$E(x,u)$};
\draw[arrow] (u1) -- (g1) node[label, above right, xshift=2pt] {$E(u,g)$};
\draw[dasharrow, bend left=20] (x1) to node[label, above, yshift=3pt] {\footnotesize trajectory $\gamma$} (g1);
% === RIGHT SIDE: Quasimetric (QRL) ===
\node[title] at (4, 4.5) {Quasimetric RL};
\node[state] (s) at (2, 2) {$s$};
\node[intermediate] (w) at (4, 3.5) {$w$};
\node[goal] (g2) at (6, 2) {$g$};
\draw[arrow] (s) -- (g2) node[label, below, yshift=-2pt] {$d^*(s,g)$};
\draw[arrow] (s) -- (w) node[label, above left, xshift=-2pt] {$d^*(s,w)$};
\draw[arrow] (w) -- (g2) node[label, above right, xshift=2pt] {$d^*(w,g)$};
\end{tikzpicture}
\caption{Intrinsic-energy JEPAs induce quasimetric structure.}
\label{fig:jepa-quasimetric}
\end{figure}
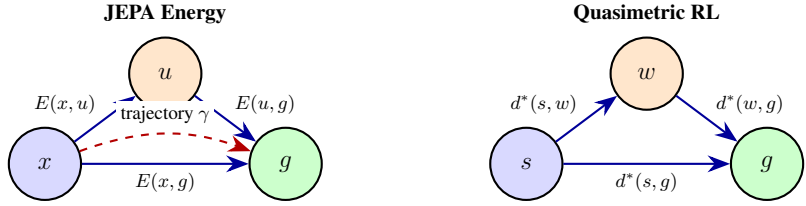

\vspace{-0.5em}

\paragraph{Goal-reaching control as intrinsic energy.}
Consider a reaching-cost problem on $\X$: for each feasible trajectory $\gamma$ from $x$ to $g$, define accumulated cost $\Act(\gamma)=\int_0^T c(\gamma(t),\dot\gamma(t))dt$ with $c\ge 0$ (continuous-time) or $\sum_t c(s_t,s_{t+1})$ (discrete-time). The optimal cost-to-go is
\[
V^\star(x,g)=\infimum_{\gamma\in\Gam(x\to g)} \Act(\gamma),
\]
which is exactly an intrinsic energy (Definition~2). This is the standard variational/optimal control object underlying dynamic programming \citep{bertsekasDP}.

\begin{corollary}[IE-JEPA $\subseteq$ QRL hypothesis class (goal-reaching setting)]
In goal-reaching problems where the optimal cost-to-go is an intrinsic energy, any IE-JEPA energy that approximates this intrinsic energy is a quasimetric cost-to-go. This places IE-JEPA energies within the same class of quasimetric value functions formalized and targeted by QRL \cite{qrl}.
\end{corollary}
\subsection*{Why symmetry fails for directed reachability ?}
% \subsection*{A minimal impossibility statement (why symmetry fails for directed reachability ?)}
To avoid implying that asymmetry is an aesthetic choice, we include the following basic obstruction.

\begin{proposition}[Symmetric Finite Energies cannot Represent Directed Reachability]\label{prop:symmetry_obstruction}
Let $R\subseteq\X\times\X$ be a directed reachability relation, where $(x,y)\in R$ means ``$y$ is reachable from $x$''.
Suppose\linebreak
$E:\X\times\X\to\bar{\mathbb{R}}^+$ satisfies:
(i) $E(x,y)<+\infty$ iff $(x,y)\in R$ (finite energy on reachable pairs),
and (ii) $E$ is symmetric: $E(x,y)=E(y,x)$.
Then $R$ must be symmetric: $(x,y)\in R \Rightarrow (y,x)\in R$.
Hence, no symmetric finite energy can encode one-way reachability.
\end{proposition}

\vspace{-1.0em}

\begin{proof}
If $(x,y)\in R$, by (i) $E(x,y)<+\infty$. By (ii) $E(y,x)=E(x,y)<+\infty$, and $(y,x)\in R$.
\end{proof}

\paragraph{Relation to value-guided JEPA planning.}
% Destrade et al.\ \cite{destrade2025valueguided} propose to shape JEPA representation spaces so that (quasi-)distances approximate negative goal-conditioned values, improving planning performance. Our scope is narrower: we do not introduce planners nor additional empirical studies; we provide a structural condition (intrinsic energy) under which a JEPA score is necessarily a quasimetric cost-to-go, directly connecting JEPA world-model scoring to the QRL value-geometry viewpoint.
\cite{destrade2025value} shape JEPA representation spaces so that (quasi-)distances approximate negative goal-conditioned values, improving planning performance. Our focus is narrower: we introduce neither planners nor new experiments, but instead identify a structural condition (intrinsic energy) under which a JEPA score is necessarily a quasimetric cost-to-go, directly linking JEPA world-model scoring to the QRL geometry perspective.

    % bACK MATTER

        % DISCUSSION
        \section{Discussion and scope}
\vspace{-0.2em}
\paragraph{What this note does \emph{not} claim.}
We do not claim that all JEPAs are quasimetrics, nor that the JEPA learning framework matches the QRL one unconditionally. 
Our equivalence is \emph{conditional} : 
it holds for the family intrinsic energies (least-action functionals). 
This is a meaningful restriction since intrinsic energies are a canonical compositional energies in physics and optimal control.

\vspace{-0.5em}

\paragraph{Why the least-action form is not ad hoc.}
The least-action definition is the minimal way to obtain :\linebreak
(i) temporal compositionality,
(ii) compatibility with irreversible systems,
and (iii) identification with cost-to-go.
Under this view, quasimetric regularization is the underlying consistency property.

\vspace{-0.5em}

\paragraph{Beyond RL.}
The intrinsic-energy view applies whenever there are admissible transformations and accumulated local efforts.
This includes (a) deformation models in diffeomorphic image registration where distances are defined as least-action energies,
and (b) reasoning structures where implications form directed relations with ordered-embedding approaches explicitly modelling directionality.

        \newpage

        % BIBLIOGRAPHY
        \bibliography{iclr2026_conference}
        \bibliographystyle{iclr2026_conference}

\end{document}